\documentclass[letterpaper, 10 pt, conference]{ieeeconf}  
\IEEEoverridecommandlockouts             
\usepackage{amssymb} %

\usepackage{bm}
\usepackage{enumerate}
\usepackage{commath}
\usepackage{graphicx}
\graphicspath{{Pictures/}}
\usepackage{amsmath}
\usepackage{mathtools}

\makeatletter
\let\MYcaption\@makecaption
\makeatother

\usepackage[font=footnotesize]{subcaption}

\makeatletter
\let\@makecaption\MYcaption
\makeatother

\usepackage{breqn}
\usepackage{cite}
\usepackage[table]{xcolor}
\usepackage{booktabs}
\usepackage{empheq}
\usepackage{amsfonts}

\usepackage{amsthm}
\usepackage{hyperref}
\usepackage{xcolor}
\usepackage{mathrsfs}
\usepackage{accents}
\usepackage{thmtools}
\usepackage{thm-restate}
\usepackage{float}
\usepackage{xcolor}
\usepackage[normalem]{ulem}

\usepackage{algorithm}
\usepackage{algpseudocode}
\usepackage{dblfloatfix}

\usepackage{etoolbox}
\usepackage{float}  %

\usepackage{comment}
\usepackage{balance}

\theoremstyle{plain}

\newtheorem{lemma}{Lemma}

\newtheorem*{theorem*}{Theorem}
\newtheorem{assumption*}{Assumption}

\declaretheorem[name=Theorem]{thm}

\theoremstyle{definition}

\newtheorem*{problem*}{Problem}

\newcommand{\myset}[1]{\mathcal{#1}}

\DeclareMathOperator*{\argmin}{argmin}

\title{ \LARGE \bf
    Secure Safety Filter: Towards  Safe Flight  Control under Sensor Attacks
}

\author{Xiao Tan$^1$, Junior Sundar$^2$, Renzo Bruzzone$^2$, Pio Ong$^1$, Willian T. Lunardi$^2$, \\ Martin Andreoni$^2$, Paulo Tabuada$^3$,  and Aaron D. Ames$^1$ %
\thanks{This work is supported by TII under project \#A6847.}
\thanks{$^1$Xiao Tan, Pio Ong, and Aaron D. Ames are with the Department of Mechanical and Civil Engineering, California Institute of Technology, Pasadena, CA 91125, USA (Email: {\tt\small xiaotan, pioong, ames@caltech.edu}).} 
\thanks{$^2$Junior Sundar, Renzo Bruzzone, Willian T. Lunardi, and Martin Andreoni are with the Secure Systems Research Center at Technology Innovation Institute, Abu Dhabi, U.A.E. (Email: {\tt\small junior.sundar, renzo.bruzzone, willian.lunardi, martin.andreoni@tii.ae}).}
\thanks{$^3$Paulo Tabuada is with the Department of Electrical and Computer Engineering at University of California, Los Angeles, CA 90095, USA (Email: {\tt\small  tabuada@ucla.edu}).} 
}

\begin{document}

\maketitle
\thispagestyle{plain}
\pagestyle{plain}

\begin{abstract}
Modern autopilot systems are prone to sensor attacks that can jeopardize flight safety. To mitigate this risk, we proposed a modular solution: the \textit{secure safety filter}, which extends the well-established control barrier function (CBF)-based safety filter to account for, and mitigate, sensor attacks. This module consists of a secure state reconstructor (which generates plausible states) and a safety filter (which computes the safe control input that is closest to the nominal one). Differing from existing work focusing on linear, noise-free systems, the proposed secure safety filter handles bounded measurement noise and, by leveraging reduced-order model techniques, is applicable to the nonlinear dynamics of drones. Software-in-the-loop simulations and drone hardware experiments demonstrate the effectiveness of the secure safety filter in rendering the system safe in the presence of sensor attacks.

  \end{abstract}

\section{Introduction}
\label{sec:intro}

Drones with autopilot capabilities are being deployed for a variety of tasks, from package delivery in urban areas to environmental surveillance in the wild. Most drones rely on onboard sensor information, such as GPS, Inertia Measurement Unit (IMU), and cameras, to achieve waypoint tracking and collision avoidance. The reliance on onboard sensors makes them vulnerable to sensor spoofing attacks. One notable example was the hijacked RQ-170 incident in which the GPS signal was believed to be maliciously manipulated \cite{rq170incident}. Many other physical attack strategies on different types of sensors have also been reported in the literature \cite{amin2010stealthy,liu2011false,kwon2013security}. For example, readings from an IMU sensor can be altered using acoustic signals \cite{trippel2017walnut}. Sensor attacks can also be initiated in the cyber domain, such is the case when using compromised hardware drivers or communication buses. Regardless of where the sensor attacks emerge, they will have a catastrophic impact on the downstream navigation and control stack: guaranteeing drone safety with a tolerance for possible sensor spoofing attacks is becoming a real and pressing challenge.

In this work, we approach the secure and safe flight control problem from a zero-trust sensor fusion perspective: we assume no \textit{a priori} trust on any onboard sensor and propose a secure safety filter module with guaranteed flight safety under possible sensor attacks from an adversarial party. The proposed scheme is shown in Fig. \ref{fig:control_diagram}. The secure safety filter consists of a safe state reconstructor and a safety filter: the former cross-examines past sensor measurements and determines plausible states at the current time. The safety filter  \cite{Ames2017,Ames2019control,wabersich2023data} removes unsafe control signals from a nominal signal, accounting for all plausible states. Using a reduced-order model, we also demonstrate how to apply the theoretical results developed for linear systems to drones. The effectiveness of our proposed secure safety filter is shown both in high-fidelity simulations and hardware experiments.

\begin{figure}[t]
\vspace{5pt}
    \centering
    \includegraphics[width=1.0\linewidth]{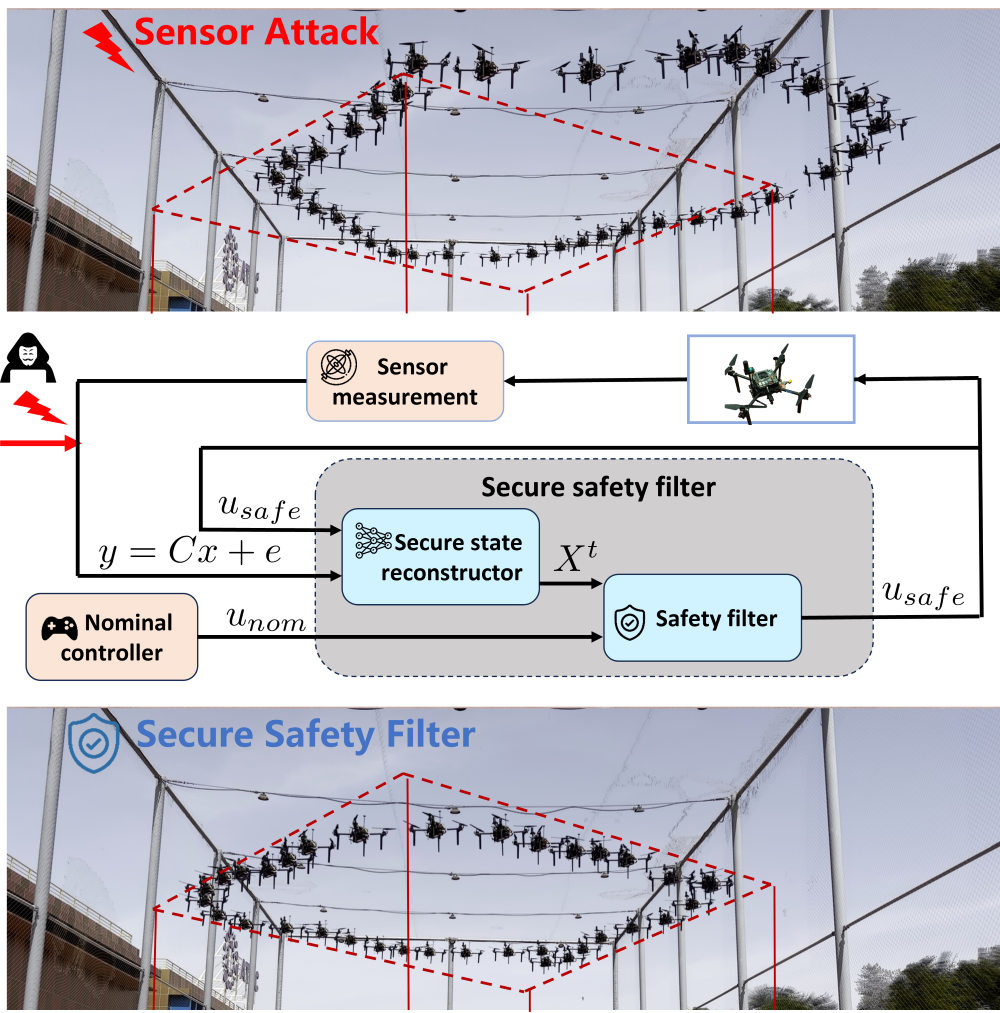}
    \caption{Diagram showing the architecture of the secure safety filter operating in the presence of sensor attacks (middle). An attacker tries to induce unsafe behaviors on the drone by injecting attack signals into sensor measurement. The secure safety filter takes past (potentially corrupted) sensor measurements and input sequences, yielding a safe control signal with the least deviation from the nominal control signal.  Hardware experiments on a drone verify the approach, with the unsafe behavior when it is under sensor attack (top) and safe behavior with the secure safety filter even when experiencing a sensor attack (bottom). }
    \label{fig:control_diagram}
    \vspace{-10pt}
\end{figure}

\begin{figure*}[b]
\small
    \begin{equation} \label{eq:matrices} \tag{4}
    \widetilde{Y}_i = \begin{bsmallmatrix}
        y_i(t-l) \\
        y_i(t-l+1) \\
        \vdots \\
        y_i(t)
    \end{bsmallmatrix},
    E_i = \begin{bsmallmatrix}
        e_i(t-l) \\
        e_i(t-l+1) \\
        \vdots \\
        e_i(t)
    \end{bsmallmatrix},
    U = \begin{bsmallmatrix}
        u(t-l) \\
        u(t-l+1) \\
        \vdots\\
        u(t-1) \\
        0
    \end{bsmallmatrix},
    \mathcal{O}_i = \begin{bsmallmatrix}
        C_i \\
        C_i A \\
        \vdots \\
        C_i A^{l}
    \end{bsmallmatrix},  F_i = \begin{bsmallmatrix}
        0 & 0 & \cdots &  0 & 0 \\
        C_iB & 0 & \cdots & 0 & 0  \\
        \vdots & & \ddots & & 0 \\
        C_iA^{l-1}B &  C_iA^{l-2}B & \cdots & C_i B & 0
    \end{bsmallmatrix},     
    \begin{aligned}
    Y_i &  = \widetilde{Y}_i - F_i U \\
      &  \text{for } i \in [p].
    \end{aligned}
\end{equation}
\end{figure*}

\setcounter{equation}{0}

\subsection{Related Works}
 Several defense strategies have been investigated in the literature to mitigate the effects of sensor attacks. One common approach in the cyber-physical security community is to conduct anomaly detection on individual sensor outputs by constantly monitoring certain statistical quantities \cite{tharayil2020sensor}. However, this defense strategy will fail if the attack signal is intelligent, for example, a replay attack \cite{zhu2013performance}. On the other hand, since onboard sensor measurements are from the same physical process, they should be consistent with system dynamics if not altered by the attacker. Leveraging this idea, several attack/fault identification algorithms are proposed in the literature for various sensor attack scenarios\cite{de2015input,zhu2013performance,smith2015covert,mo2010false}. 
 
 In this work, we focus on the sensor spoofing attack model, that is, the output of the attacked sensors can be arbitrarily modified by the attacker. Existing works \cite{fawzi2014secure,shoukry2015event} have exact characterization on the worst-case number of sensors that can be attacked while the true state can still be recovered and several algorithms to reconstruct the state. Recently, severe sensor attack scenarios have been considered in \cite{tan2024safety}, in which case the ground-true state may not be able to be uniquely determined. These plausible states are characterized by combinatorial algorithms and fed into a safety filter to guarantee closed-loop safety. A computationally efficient algorithm for determining the safety condition is investigated \cite{tan2025computationally}. However, all aforementioned results are developed without considering the effects of measurement noise on those attack-free sensors.

Until recently, only a few papers have started to look into the closed-loop system safety guarantee under various sensor attack scenarios. In \cite{lin2022plug}, a set of safeguarded sensors (i.e., sensors that can not be attacked) are assumed available for the design of the so-called secondary controller for safety. \cite{zhang2022safe} assumes the existence of several groups of sensors and the sensor faults are detected per sensor groups instead of individual sensors. In a recent work \cite{arnstrom2024data}, the authors explicitly analyze the effects of stealthy sensor attacks on deactivating existing safety filters. The assumptions made by these works do not hold true under the aforementioned sensor spoofing attack model that we adopt.

\subsection{Contributions}

In this work, we propose the \emph{secure safety filter}: an extension of the traditional safety filter that ensures safety for systems under sensor spoofing attacks. 
The proposed formulation has the following features:
\begin{enumerate}
    \item The secure safety filter accounts for bounded measurement noise on attack-free sensors. Our analysis shows this results in ball-like inflation of the plausible states identified in previous works.
    \item The secure safety filter applies a robust control barrier function condition that handles the inflated plausible states to formally guarantee safety.
    \item The secure safety filter can be applied to nonlinear systems leveraging reduced-order models. We demonstrate its effectiveness on a drone system.
\end{enumerate}

The paper is organized as follows. Section~\ref{sec:prem} recalls existing results in secure state reconstruction and safety filter design. Section~\ref{sec:secfilter} investigates the secure safety filter design in the presence of measurement noise. Details on how to apply the proposed secure safety filter to nonlinear drone dynamics and inject sensor attacks on the drone are presented in Section~\ref{sec:drone_dyn}. SITL simulation and experimental results are reported in Section~\ref{sec:results}. We then conclude the paper in Section~\ref{sec:conclusion}.

\section{Preliminaries}
\label{sec:prem}

Consider a discrete-time linear system model under sensor spoofing attacks:
\begin{equation}\label{eq:system}
    \begin{aligned} 
    x(t+1) & = Ax(t) + Bu(t),\\
     y(t)  & = Cx(t) + e(t), 
     \end{aligned}
\end{equation}
where $x\in \mathbb{R}^n$, $u\in \mathbb{R}^m$, $y\in \mathbb{R}^p$, and $e\in\mathbb{R}^p$ are the state, the input, the sensor measurement signals, and the sensor attack signals, respectively. In this work, each entry of the measurement $y$ is taken as a sensor output. Sensor $i$ is under \textit{sensor spoofing attack} if there exists time $t$ such that $e_i(t)\neq 0$. Our main objective is to endow the system with the ability to maintain safety in the presence of (up to) $s$ spoofing attacks. Sensor spoofing attack is a generic attack model since no additional restrictions are imposed upon the attacker.
\subsection{Secure state reconstruction (SSR)}

Secure state reconstruction~\cite{fawzi2014secure,shoukry2015event,tan2024safety}, as a countermeasure to spoofing attacks studied in the literature, relies on past input-output data from the system. See Fig.\ref{fig:control_diagram} for an illustration. At time instant $t$, we consider the most recent input-output data of length $l+1$: 
\begin{equation}
\begin{aligned}
    \mathcal{D}_t = (& u(t-l),u(t-l+1), \dots, u(t-1),   \\ & y(t-l),y(t-l+1), \dots, y(t))\in \mathbb{R}^{ml + p(l+1)}
\end{aligned}   
\end{equation}
with $l+1\geq n$. The true ``initial" state of the system, $x(t-l)$, is consistent with the data $\mathcal{D}_t $ when the output data from the $s$ attacked sensors is ignored. Mathematically, the initial state $x(t-l)$ is a member of the set $\myset{X}_{t}^{t-l,\Gamma} = \{x\in \mathbb{R}^n: \mathcal{O}_{i}x = Y_{i} \textup{ for } i \in \Gamma \}$ where $\Gamma$ is the set of indices of attack-free sensors, and the matrices $\mathcal{O}_i$ and $Y_i$ are given in \eqref{eq:matrices} at the bottom of the current page. Since the indices of intact sensors are not known, secure state reconstruction enumerates through all possible sensor combinations $\Gamma\in\mathbb{C}_{p}^{p-s}$, where $\mathbb{C}_{p}^{p-s}$ consists of all subsets of $[p]=\{1,2,\ldots,p\}$ with size $(p-s)$, to produce the set of all \textit{plausible initial states}:
    \begin{equation} \label{eq:X0set}
        \myset{X}_t^{t-l} = \bigcup_{\Gamma = \{i_1, \ldots, i_{p-s}\} \in \mathbb{C}_{p}^{p-s}} 
 \myset{X}_{t}^{t-l,\Gamma}. 
    \end{equation}

\setcounter{equation}{4}

  We can then 
forward propagate this set to obtain the set of all plausible states  at the current time $t$:
\begin{multline}
    \label{eq:X_t}
    \myset{X}_t^{t} = A^{l}( \myset{X}_{t}^{t-l}) + A^{l-1}Bu(t-l) +  \ldots + Bu(t-1).
\end{multline}

Throughout this work, we restrict our discussion to the case that the system \eqref{eq:system} is \emph{$s$-sparse observable}, i.e., the system remains observable after removing measurements from $s$ arbitrary sensors. The set $\myset{X}_t^t $ contains finitely many plausible states in this scenario, as proved in~\cite{tan2024safety} where more thorough discussions of secure state reconstruction can be found.

\subsection{Safety filter design}

We recall the safety filter design using the control barrier function (CBF) framework \cite{Ames2017, agrawal2017discrete}. Suppose that the state of system \eqref{eq:system} needs to satisfy a safety constraint $x(t)\in \myset{C} = \{x\in \mathbb{R}^n:  h(x):= Hx + q \geq 0 \}$ for all time. Given a nominal control signal $t\mapsto u_{\textup{nom}}(t)$ with current state $x(t)$, the \emph{CBF safety filter} \cite{Ames2019control} is implemented as: 
\begin{align}
         u_{\text{safe}}(t)  = \quad  & \argmin_u \| u - u_{\textup{nom}}(t) \|^2  \label{eq:qp_control_cost}\\
       \mathrm{s.t.} \quad  &   H(Ax(t) + Bu) + q  \geq (1-\gamma) (Hx(t) + q), \label{eq:qp_control_constraint}
\end{align}
where $ 0< \gamma <1$ is a constant scalar. The filter effectively picks an input $u$ with the least deviation, pointwise, from the nominal one to satisfy the CBF condition in \eqref{eq:qp_control_constraint}. From CBF theory, this safety filter produces state trajectories contained in $\mathcal{C}$, i.e., the system is safe.

\section{Secure safety filter design}
\label{sec:secfilter}
This section presents our proposed secure safety filter design for systems with measurement noise. The design comprises a secure state reconstruction submodule and a robust safety filter submodule. We first provide a theoretical analysis of each submodule, and then summarize the overall secure safety filter design in a concise algorithm and formalize its theoretical guarantees in a theorem. 

\subsection{SSR under measurement noise}
In robotic systems, measurement noise is inevitable, but this has not been considered in existing secure state reconstruction literature \cite{fawzi2014secure,shoukry2015event,tan2024safety}. To account for this effect, we consider the following system model:
\begin{equation}\label{eq:disturbed_system}
    \begin{aligned}
    x(t+1) & = Ax(t) + B(u(t)),\\
     y(t)  & = Cx(t) + e(t) + d(t), 
     \end{aligned}
\end{equation}
where $d$ is the measurement noise. We assume that the measurement noise is bounded in all sensors:  
\begin{equation} \label{eq:disturbance_bound}
   \|d(t) \|_{\infty} \leq d_{{\max}} \text{ for any } t.
\end{equation}

With measurement noises, the input-output history $\mathcal{D}_t $ satisfies the following equality:
\begin{equation} \label{eq:disturbed_system_compact2}
     Y_i = \mathcal{O}_i x(t-l)+ E_i + D_{i}, \ i\in [p].
\end{equation}
where $D_{i} := [d_{i}(t-l), d_{i}(t-l+1), \ldots, d_{i}(t)]^\top$. Since the $D_{i}$ is unknown, a state $x$ is a \textit{plausible initial state at time $t-l$ under measurement noise} when equation \eqref{eq:disturbed_system_compact2} holds with $E_i = 0$ for at least $p-s$ sensors and with a measurement noise satisfying its bound $\| D_{i}\|_{\infty}\leq d_{{\max}}$.
We establish the following result.
\begin{lemma} \label{lem:X0d}
    Let $\myset{X}_{t,d}^{t-l}$ be the set of plausible initial states under measurement noise. Then we have:
    \begin{equation} \label{eq:X0set_disturb}
        \myset{X}_{t,d}^{t-l} = \bigcup_{\Gamma = \{i_1, \ldots, i_{p-s}\} \in \mathbb{C}_{p}^{p-s}} 
 \myset{X}_{t,d}^{t-l,\Gamma} 
    \end{equation}
  with   $$ \myset{X}_{t,d}^{t-l,\Gamma} = \{x\in \mathbb{R}^n: \| \mathcal{O}_{i}x - Y_{i} \|_{\infty} \leq d_{{\max}} \textup{ for } i \in \Gamma \}.$$
\end{lemma}
\begin{proof}

    For any $x\in \bigcup_{\Gamma \in \mathbb{C}_{p}^{p-s}} 
 \myset{X}_{t,d}^{t-l,\Gamma}  $, it must belong to $\myset{X}_{t,d}^{t-l,\Gamma'} $ with a specific index set $\Gamma'$. That is, the state $x$ satisfies $\| \mathcal{O}_{i}x - Y_{i} \|_{\infty} \leq d_{{\max}} \textup{ for } i \in \Gamma'$. For this state, we also know \eqref{eq:disturbed_system_compact2} holds with measurement noise $D_{i} = Y_{i}-\mathcal{O}_{i}x$ and with $E_i=0$ for all $p-s$ sensors indexed in $\Gamma'$. Since $\|D_{i}\|_\infty \leq d_{{\max}}$, the state is a plausible initial state, i.e., $x\in \myset{X}_{t,d}^{t-l}$. 

    For any $x\in \myset{X}_{t,d}^{t-l}$, denote the set of attack-free sensors with size $p-s$ by $\Gamma'$. For $i\in\Gamma'$, \eqref{eq:disturbed_system_compact2} must hold with $E_i = 0$, so we have $D_{i} = Y_{i}-\mathcal{O}_{i}x$. Since $\| D_{i}\|_{\infty}\leq d_{{\max}}$, it follows that  $x\in \myset{X}_{t,d}^{t-l,\Gamma'}\subseteq \bigcup_{\Gamma = \in  \mathbb{C}_{p}^{p-s}}  \myset{X}_{t,d}^{t-l,\Gamma} $. 
\end{proof}

Lemma \ref{lem:X0d} accurately characterizes the plausible initial state set. Equation \eqref{eq:X0set_disturb} offers a decomposition of the set into smaller and more manageable sets. In particular, for each index combination $\Gamma$, two plausible states $x_1,x_2\in\myset{X}_{t,d}^{t-l,\Gamma}$ must satisfy:
\begin{equation}
\label{eq:tri_ineq}
    \|\mathcal{O}_i(x_1-x_2)\|\leq \|\mathcal{O}_ix_1- Y_i\|+\|\mathcal{O}_ix_2-Y_i\|
\end{equation}
from the triangle inequality. Therefore, if we can find a plausible state in the set $x_d^{t-l,\Gamma}\in\myset{X}_{t,d}^{t-l,\Gamma}$, plausible states associated to the sensor combination $\Gamma$ can be bounded. 

  We propose the following linear program for finding a plausible state $x_d^{t-l,\Gamma}\in\myset{X}_{t,d}^{t-l,\Gamma}$ for a sensor combination $\Gamma \in \mathbb{C}_{p}^{p-s} $:
\begin{equation} \label{eq:lp_x_d}
    \begin{aligned}
        (x_d^{t-l,\Gamma}, \underline{d}^{\Gamma}) = & \argmin_{x \in \mathbb{R}^n, \  d^\Gamma\in \mathbb{R}} \quad d^\Gamma \\
        \text{ subject to } & -d^\Gamma \bm{1} \leq  \mathcal{O}_{\Gamma}x - Y_{\Gamma} \leq d^\Gamma \bm{1}
    \end{aligned}
\end{equation}
where $\mathcal{O}_{\Gamma} := [\mathcal{O}_{i_1}; \ldots; \mathcal{O}_{i_{p-s}}]$ and $Y_{\Gamma} := [Y_{i_1}; \ldots; Y_{i_{p-s}}]$. Here the solution $(x_d^{t-l,\Gamma}, \underline{d}^{\Gamma})$ always exists, and if $\underline{d}^{\Gamma}>d_{\max}$, there are no plausible states associated with the sensor combination $\Gamma$. On the other hand, if $\underline{d}^{\Gamma}\leq d_{{\max}}$, we derive from \eqref{eq:tri_ineq} that:
 \begin{equation} \label{eq:X_td_Gamma}
     \myset{X}_{t,d}^{t-l,\Gamma} \subseteq 
           \{x\in \mathbb{R}^n:  \| \mathcal{O}_{\Gamma} (x-x_d^{t-l,\Gamma})\|_{\infty}\leq d_e^{\Gamma} \}
 \end{equation}
where $d_e^{\Gamma}: = d_{{\max}} + \underline{d}^{\Gamma}$. 
In addition, we can derive a more conservative but easier set representation:
$$
\myset{X}_{t,d}^{t-l,\Gamma} \subseteq  \{ x_{d}^{t-l,\Gamma}\}+ \{ e\in \mathbb{R}^n: \| e\|_{\infty}\leq d_e^{\Gamma} /m(\mathcal{O}_{\Gamma}) \},
$$ 
where $m(\mathcal{O}_{\Gamma}) = \inf_{\| a\|_{\infty}= 1}  \| \mathcal{O}_{\Gamma} a \|_{\infty}$ is the minimum modulus of the matrix $\mathcal{O}_{\Gamma}$ with respect to the $\infty$-norm.  

The current-time plausible state set $\myset{X}_{t,d}^t$ can be computed by forward propagating the set   $\myset{X}_{t,d}^{t-l} $ similar to \eqref{eq:X_t}. Furthermore, one derives that $\myset{X}_{t,d}^t = \bigcup_{\Gamma  \in \mathbb{C}_{p}^{p-s}}  \myset{X}_{t,d}^{t,\Gamma} $. For the case $d_e^{\Gamma} \leq 2d_{{\max}}$: 
\begin{equation} \label{eq:X_td_tGamma}
     \hspace{-2mm}\myset{X}_{t,d}^{t,\Gamma} \subseteq \{ x_{d}^{t,\Gamma}\}+ \{ e\in \mathbb{R}^n: \| e\|_{\infty}\leq \| A\|_{\infty}^l d_e^{\Gamma} /m(\mathcal{O}_{\Gamma}) \}
\end{equation}
with $x_{d}^{t,\Gamma} = A^l x_{d}^{t-l,\Gamma}  + A^{l-1}Bu(t-l) +  \ldots + Bu(t-1)$; for the case $d_e^{\Gamma} < 0$,  $\myset{X}_{t,d}^{t,\Gamma} = \emptyset$.  Intuitively, the plausible state set $\myset{X}_{t,d}^{t} $ gets inflated from a set of points as in $\myset{X}_{t}^{t} $ to a set of ball-like regions in terms of $\infty$-norm distance.

\subsection{Robust safety filter design}
To guarantee the safety of the system, we need to account for all plausible states at the current time. Now, the discrete-time CBF condition becomes:
\begin{equation} \label{eq:cbf_cond_under_disturbance}
    H(Ax+ Bu) + q \geq (1-\gamma)(Hx + q), \forall x\in     \myset{X}_{t,d}^{t}.
\end{equation}
As we have over-approximated the set $\myset{X}_{t,d}^{t}$ with a union of $\infty$-norm balls, our strategy involves rewriting the constraint into multiple ones.
For each set of sensors $\Gamma\in \mathbb{C}_{p}^{p-s}$ with $d_e^{\Gamma} \geq 0$, the respective CBF condition is:
\begin{equation}
    \begin{aligned}
        HBu + H(A - (1-\gamma)I)x + \gamma q \geq 0 \quad \forall x\in   \myset{X}_{t,d}^{t,\Gamma} 
    \end{aligned}
\end{equation}
This way, we may take the robust CBF approach \cite{jankovic2018robust} and derive a sufficient condition:
\begin{equation} \label{eq:cbf_cond_under_disturbance_robust}
    HBu + H(A - (1-\gamma)I)x_d^{t,\Gamma} + \gamma q  - \Delta_{\Gamma} \bm{1} \geq 0 
\end{equation}
where $\Delta_{\Gamma} := \| H(A - (1-\gamma)I)\|_{\infty} \| A\|_{\infty}^l d_e^{\Gamma} /m(\mathcal{O}_{\Gamma}) $, to avoid considering an infinite number of states in our constraint.

The end result is the secure safety filter with constraints parametrized by a finite number of states:
\begin{equation} \label{eq:secure_safety_filter}
    \begin{aligned}
        u_{\textup{safe}}(t) =  \quad &  \argmin_{u} \ \|  u - u_{\text{nom}}(t)\|^2 \\
        \mathrm{s.t.} \quad  &  HBu + H(A - (1-\gamma)I)x_d^{t,\Gamma} + \gamma q  - \Delta_{\Gamma} \bm{1} \geq 0 
    \end{aligned}
\end{equation}
for all $\Gamma\in \mathbb{C}_{p}^{p-s}$ with $d_e^{\Gamma} \geq 0$.

Following the above analysis, an algorithmic description of the secure safety filter is given in Alg. \ref{alg:secure_safety_filter}. The safety guarantee of the proposed secure safety filter is formalized below. The proof follows from Lemma~\ref{lem:X0d} and standard analysis of CBF-based safety filter design and is thus omitted.

\begin{thm}
    Suppose that the secure safety filter \eqref{eq:secure_safety_filter} is always feasible. Then, the control sequence $u_{\textup{safe}}(t), t\in \mathbb{N},$ will render the system safe.
\end{thm}

\begin{algorithm}[h]
\caption{Secure safety filter}
\label{alg:secure_safety_filter}
\begin{algorithmic}[1]
\Require{ nominal control input $u_{\textup{nom}}(t)$ and input-output data $\mathcal{D}_t$ of length $l+1\geq n$ }
\Ensure{ safe control input  $u_{\textup{safe}}(t)$}
\For{each  $\Gamma \in \mathbb{C}_{p}^{p-s} $}
\State Compute $x_d^{t-l,\Gamma}, \underline{d}^{\Gamma}$ according to \eqref{eq:lp_x_d}
\If{$ d_{{\max}} \geq \underline{d}^{\Gamma}$}
\State Forward propagate $x_d^{t-l,\Gamma}$ to obtain $x_d^{t,\Gamma}$ 
\State Collect the robust CBF condition \eqref{eq:cbf_cond_under_disturbance_robust} 
\EndIf
\EndFor
\State Formulate and solve the QP problem \eqref{eq:secure_safety_filter} with collected CBF conditions from Step 5
\State \Return the safe input $u_{\textup{safe}}(t)$
\end{algorithmic}
\end{algorithm}

The feasibility assumption in the above result is challenging to verify in general. For certain sensor spoofing attacks, the system is doomed to take the risk of being unsafe regardless what control is in use.  In \cite{tan2024safety}, a thorough analysis on the feasibility in the absence of measurement noise was provided. Extending that result here is not trivial and requires future endeavors. In the following drone experiments with geofencing constraints, we switch to a safe but conservative zero velocity command when the secure safety filter is infeasible.

\section{Drone dynamics and sensor attacks}
\label{sec:drone_dyn}

This section outlines the application of the secure safety filter to drones via reduced-order models. 

\subsection{Reduced-order model for drones}
Quadrotor drones are intrinsically nonlinear dynamical systems \cite{mahony2012multirotor}. To apply our proposed secure safety filter, we must first develop a linear model for the drone. One common approach for approximating nonlinear systems using a linear model is local linearization of the dynamics around an equilibrium point (hovering state in the quadrotor case). However, such model does not work well with safe navigation task, because the linear approximation is only valid near the equilibrium point.

Instead,
practical quadrotor control systems usually has a cascaded control architecture with a faster inner attitude control loop and a slower outer velocity and position tracking loop \cite{meier2015px4}. Existing off-the-shelf controllers, such as the open-source PX4 Autopilot, can track velocity set-points well. For the purpose of safe navigation tasks, we can rely on the inner loop to handle nonlinearities and approximate the dynamics of a drone navigating in a horizontal plane using the following reduced-order model:
\begin{equation} \label{eq:reduced_order_model}
    \begin{bmatrix}
        \dot{x}_1 \\
        \dot{v}_1 \\
        \dot{x}_2 \\
        \dot{v}_2
    \end{bmatrix} = \begin{bmatrix}
        0 & 1 & 0 & 0 \\
        0 & -\frac{1}{\tau_s} & 0 & 0 \\
        0 & 0 & 0 & 1 \\
        0 & 0 & 0 & -\frac{1}{\tau_s}
    \end{bmatrix} \begin{bmatrix}
        x_1 \\
        v_1 \\
        x_2 \\
        v_2
    \end{bmatrix} + \begin{bmatrix}
        0  & 0\\
        \frac{1}{\tau_s}  & 0\\
        0 & 0 \\
        0 & \frac{1}{\tau_s}
    \end{bmatrix} \begin{bmatrix}
        u_1 \\
        u_2 
    \end{bmatrix}
\end{equation}
where $\tau_s$ is a time constant, $x_1, x_2$, $ v_1, v_2$, and $u_1, u_2$ are positions, velocities, velocity set-points (commands) of $\bm{x},\bm{y}$-axes in a horizontal plane, respectively. This reduced-order model simplifies/abstracts the nonlinear velocity tracking dynamics (the inner loop) as a first-order system with time constant $\tau_s$ in seconds. The parameter $\tau_s$ indicates how fast the actual velocity tracks the velocity command and is determined before our flight simulations and tests. Our discrete-time linear system is then computed based on the discretization of the continuous-time system~\eqref{eq:reduced_order_model} using a zero-order hold (ZOH) of sampling time $T_s = 0.05$ seconds.

\subsection{Measurement model and sensor attack emulation}

As mentioned above, the drone motion is limited to two-dimensions along the $\bm{x}$ and $\bm{y}$ axes, while the altitude ($\bm{z}$-position) is assumed to be constant. In both our simulation and real experiments, the reduced-order model is assumed to have access to two copies of  $\bm{x}$-position ($x_1$), $\bm{y}$-position ($x_2$), $\bm{x}$-velocity ($v_1$), and $\bm{y}$-velocity ($v_2$) of the drone. Mathematically, the measurement model is:
\begin{equation} \label{eq:measurement_model_with_ROM}
    y(t)  = Cx(t) + e(t) + d(t) \text{ with } C = 
    \begin{bsmallmatrix}
        1 & 0 & 0 & 0 \\
        0 & 1 & 0 & 0 \\
        0 & 0 & 1 & 0 \\
        0 & 0 & 0 & 1 \\
        1 & 0 & 0 & 0 \\
        0 & 1 & 0 & 0 \\
        0 & 0 & 1 & 0 \\
        0 & 0 & 0 & 1 
    \end{bsmallmatrix}.
\end{equation}

For implementation, although these measurements can be directly obtained from onboard sensors like the IMU and GPS, the existing flight control unit usually already has a built-in Extended Kalman Filter (EKF) module that provides more accurate position and attitude estimates. This EKF state estimate is then fed into the flight controller and navigation stack, which then computes the commands for actuators. See PX4 software architecture diagram \cite{meier2015px4} for an example. For the purpose of testing and validation of our approach, the measurements from \eqref{eq:measurement_model_with_ROM} are taken from (part of) the outputs of two EKFs. As we mentioned in Section~\ref{sec:intro}, sensor attack signals can be injected into individual sensors or injected into the internal communication stack. In our test scenario, the attack is directly injected in the communication channel between the EKF module and the control stack. We note that our proposed secure safety filter 
does not need to distinguish between these two attack injection scenarios as both will affect the EKF information that influences downstream control and navigation tasks.

The sensor attacks are emulated by injecting different types of attack signals to manipulate the EKF outputs. See Fig.~\ref{fig:sitl-experiment} and Fig.~\ref{fig:hardware-experiment} for how information flows in the SITL simulation and hardware experiments. The types of attacks considered and implemented are:

\begin{enumerate}
    \item \textbf{Constant Value} - Set the measurement to a constant.
    \item \textbf{Noise} - Add a Gaussian noise to the measurement.
    \item \textbf{Scale} - Scale the measurement by a constant.
    \item \textbf{Shift} - Shift the measurement by a constant.
\end{enumerate}

We note that the reduced-order model with dynamics in \eqref{eq:reduced_order_model} and measurements in~\eqref{eq:measurement_model_with_ROM} is $1$-sparse observable, so the secure safety filter can handle any attack on one of the $8$ measurements and keeps the system safe whenever the online QP \eqref{eq:secure_safety_filter} is always feasible.

\begin{figure}[t!]
    \centering
    \includegraphics[width=0.6\linewidth]{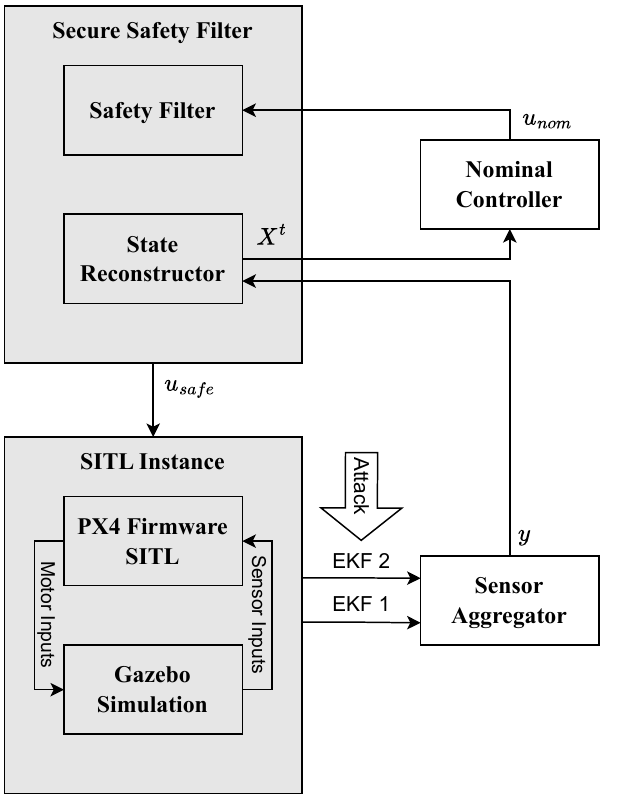}
    \caption{Diagram of Software-in-the-Loop (SITL) experiment set-up}
    \label{fig:sitl-experiment}
\end{figure}

\begin{figure}[t!]
    \centering
    \includegraphics[width=\linewidth]{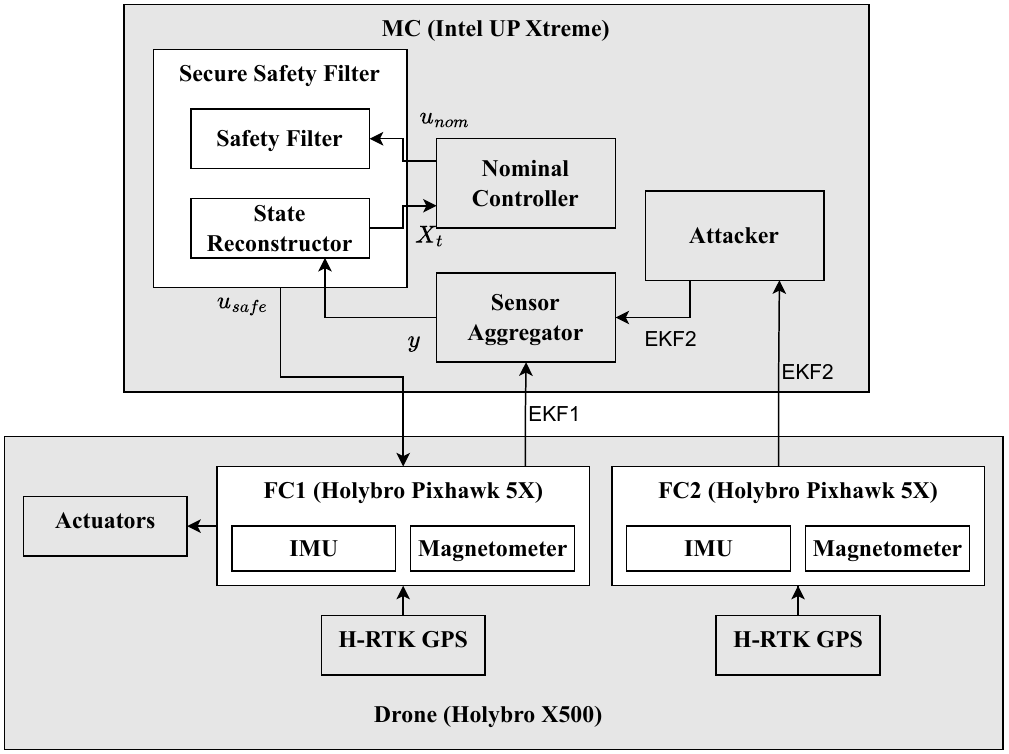}
    \caption{Diagram of hardware experiment set-up}
    \label{fig:hardware-experiment}
\end{figure}

\section{Simulation and Experiment Results}
\label{sec:results}

\subsection{Software-in-the-Loop Experiments}\label{h:sitl-experiment}

We report the Software-in-the-Loop (SITL) implementation results to verify and validate our proposed secure safety filter. Our implementation is based on ROS2 middle-ware for its distributed nature, modularity, and integration with the controller of choice, the open-source PX4-Autopilot.

PX4-Autopilot is packaged with a SITL build. The SITL simulation launches a simulated Holybro X500 drone in a Gazebo Garden simulation environment. All the onboard sensors, actuators, and drone dynamics are simulated in the high-fidelity simulation and run on a desktop computer with a 13th Gen Intel i9 and integrated graphics. Our SITL build of PX4-Autopilot is available in the organization's open-source repository. Our controller and secure-safety filter code are publicly available\footnote{Our simulation code is publicly available at 
\url{https://github.com/tiiuae/px4-secure-state-reconstruction}.}

\begin{figure*}[t]
    \centering
    	\begin{subfigure}[t]{0.3\linewidth}
		 \includegraphics[width=\linewidth]{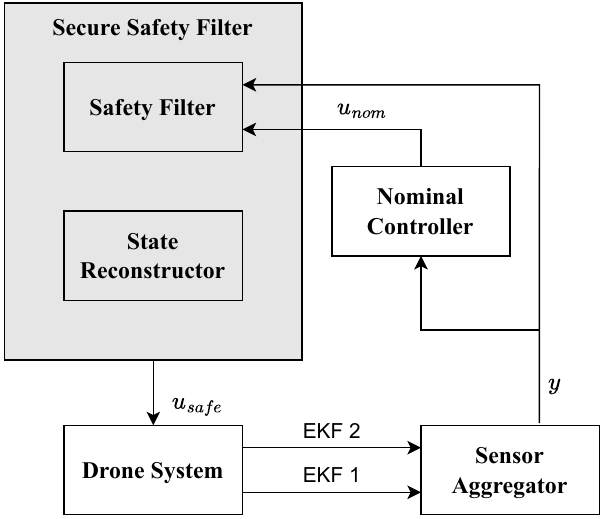}
		\caption{ Phase 1 }   
	\end{subfigure} \hspace{1mm}
	\begin{subfigure}[t]{0.3\linewidth}
		   \includegraphics[width=\linewidth]{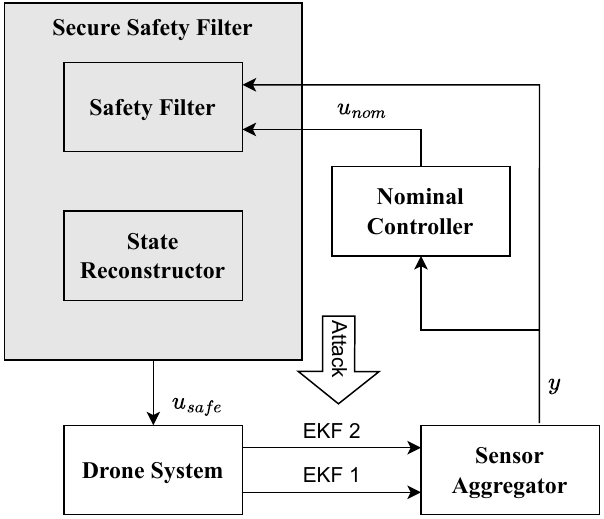}
		\caption{ Phase 2}
	\end{subfigure} \hspace{1mm}
    	\begin{subfigure}[t]{0.3\linewidth}
		   \includegraphics[width=\linewidth]{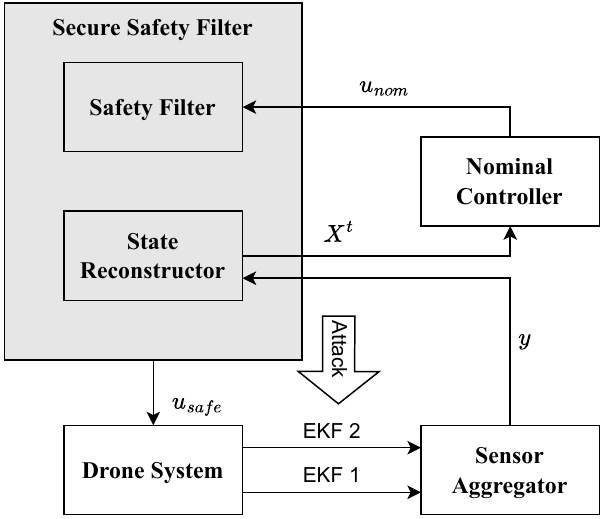}
		\caption{ Phase 3}
	\end{subfigure}
    \caption{\textbf{Phase 1} - Drone flying with a nominal controller together with a safety filter; \textbf{Phase 2} - Attack initiated on the data entering the Sensor Aggregator; \textbf{Phase 3} - Secure safety filter active and safe-guarding the nominal controller.}
    \label{fig:three-phases}
\end{figure*}

\begin{figure*}[t]
    \centering
    \includegraphics[width=\linewidth]{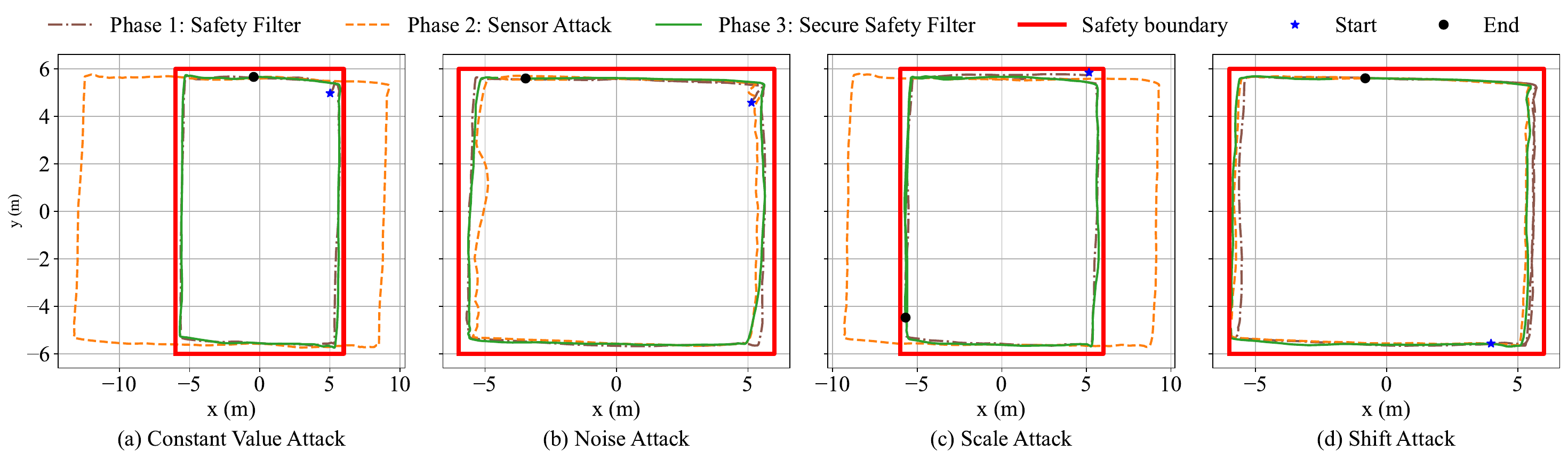}
    \caption{Quadrotor trajectories in different phases of the SITL simulation under four different attacks. For easier illustration, all the attacks are performed on $y_1$ or $y_5$ measuring $\bm{x}$-axis position. \textbf{Case (a) Constant value attack} - the measurement  is set to be constant $2$; \textbf{Case (b) Noise attack} - A Gaussian noise with zero mean and standard deviation $0.5$ is added to the measurement; \textbf{Case (c) Scale attack} -  the measurement is set to be $0.2x_1$;  \textbf{Case (d) Shift attack} - the measurement is set to be $x_1 + 0.5$. }
    \label{fig:sitl-results}
    \vspace{-10pt}
\end{figure*}

Figure \ref{fig:sitl-experiment} shows an overall diagram for our SITL implementation. Outside of the SITL Instance, each box is implemented as a ROS 2 node, and the arrows represent the information flow among these nodes. 
In the test, we activate the functionality of different nodes through three sequential phases, shown in Fig.~\ref{fig:three-phases}. Namely, there are the following three phases of experiments: 
\begin{description}
    \item[\textbf{Phase 1:}] $\quad$ The SITL drone instance, a sensor aggregator,  a nominal controller, and a safety filter are activated.
    \item[\textbf{Phase 2:}]  $\quad$ The attacker that injects attack signals on one of the EKF outputs is activated.
    \item[\textbf{Phase 3:}]  $\quad$ The secure safety filter is initialized to mitigate sensor attacks, i.e., to provide guaranteed safety under sensor attack. 
\end{description}

In the tests, the drone is tasked to continuously visit four locations $(5.7, 5.7)\rightarrow(-5.7, 5.7)\rightarrow(-5.7, -5.7)\rightarrow(5.7, -5.7)$ in a cyclic manner. The safety condition is to keep the drone flying in a square region:
\begin{equation}
    x_1 \in [-6, 6], 
    x_2 \in [-6, 6].
\end{equation}
We test four types of attacks as described in Sec.~\ref{sec:drone_dyn}.B and collect the trajectories of the drone during the three different phases, as shown in Fig.~\ref{fig:sitl-results}. 

We observe that, in all four attack cases, the attacks degrade the performance of the drone and even induce violations of the safety constraint even though a safety filter is in effect. In particular, the noise attack does not cause a safety problem because of the robustness of the safety filter. The magnitude of the shift attack is small, so the drone stays safe thanks to the buffering zone. In these two cases, the attacker fails to induce unsafe behavior of the drone. Safety violations occur in the constant-value attack case (Case (a)) and the scale attack case (Case (c)). In those cases, thanks to our secure safety filter, the drone regains safety in Phase 3 and successfully resumes its behavior from before the attack.

\begin{figure}[b!]
    \centering
   \hspace{-30pt} \includegraphics[width=0.8\linewidth]{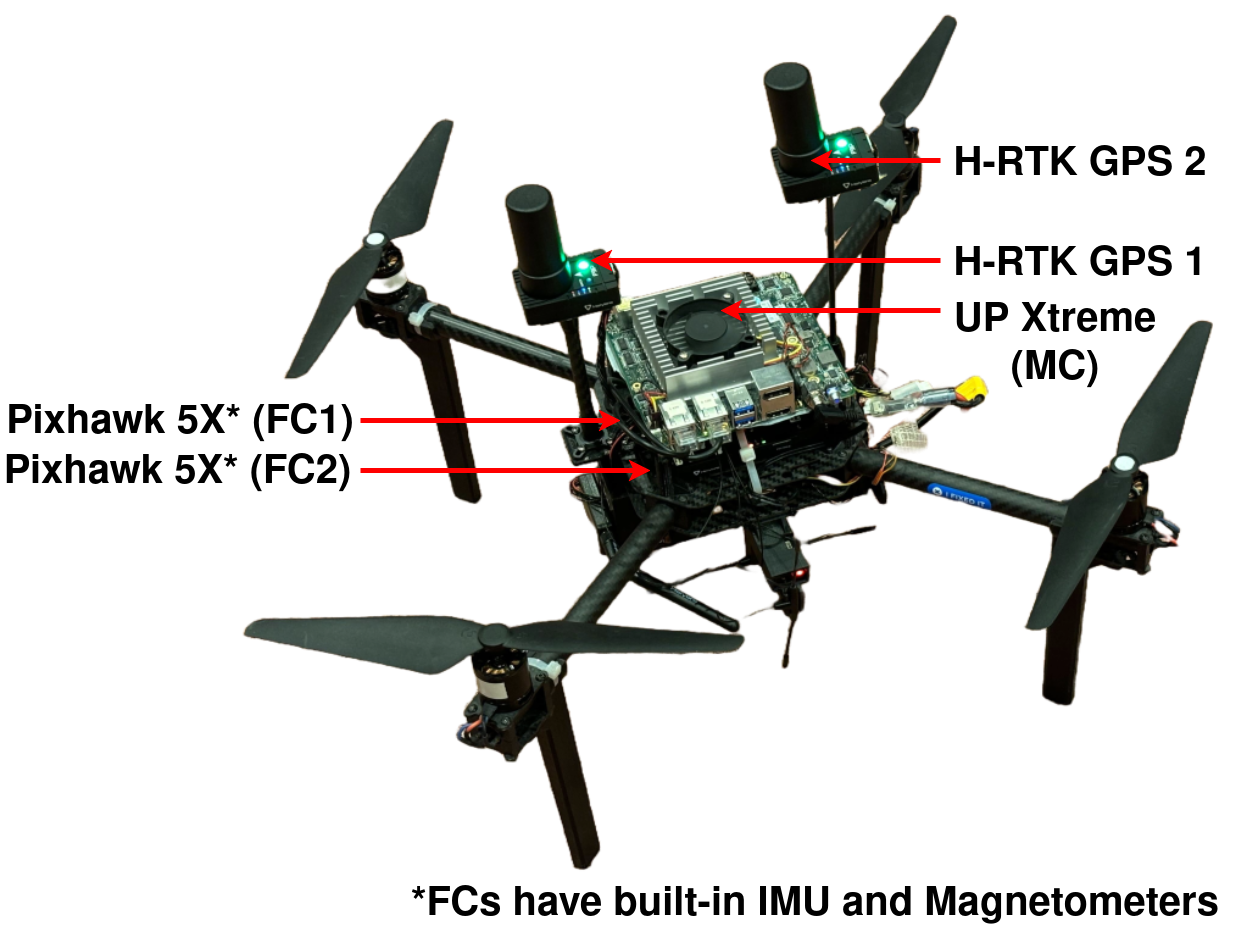}
    \caption{Holybro X500 quadrotor test platform. Two Pixhawk 5X microcontrollers, which act as Flight Controllers (FCs), and one Intel UP Xtreme mission computer (MC) are mounted onboard. The two FCs are connected to an independent set of sensors--H-RTK GPS, IMU, magnetometer, and barometer. The MC runs our customized modules including the secure safety filter, the nominal controller, the sensor aggregator, and the attacker.  }
    \label{fig:drone}
\end{figure}

\begin{figure*}[t]
    \centering
    \includegraphics[width=\linewidth]{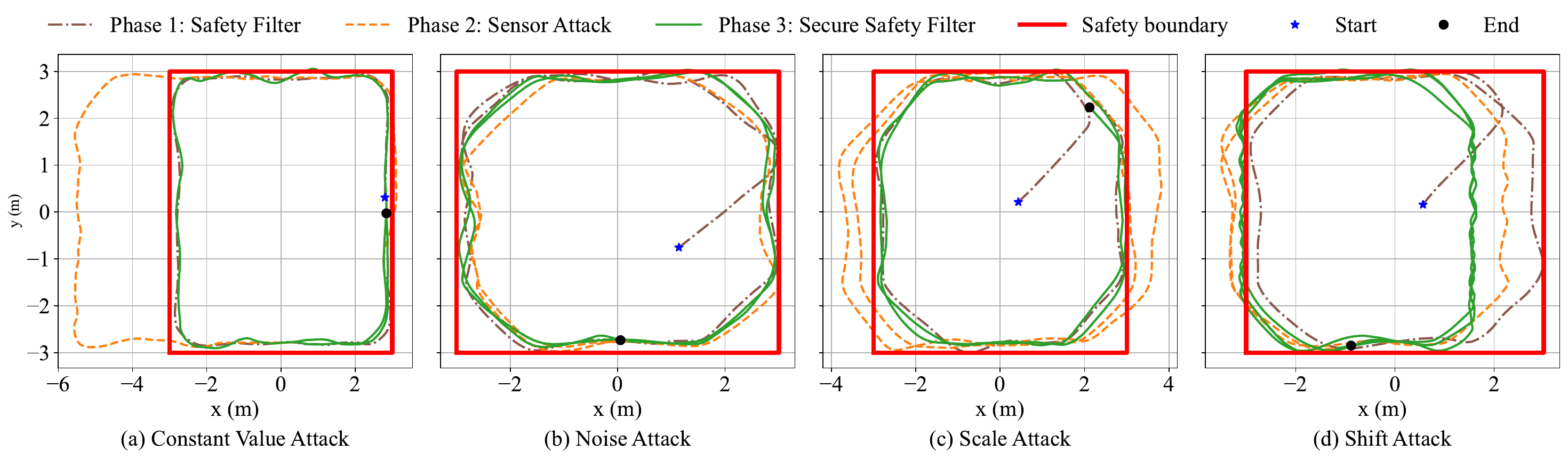}
    \begin{subfigure}[t]{\linewidth}
            \setcounter{subfigure}{4} 
         \includegraphics[width=1\linewidth]{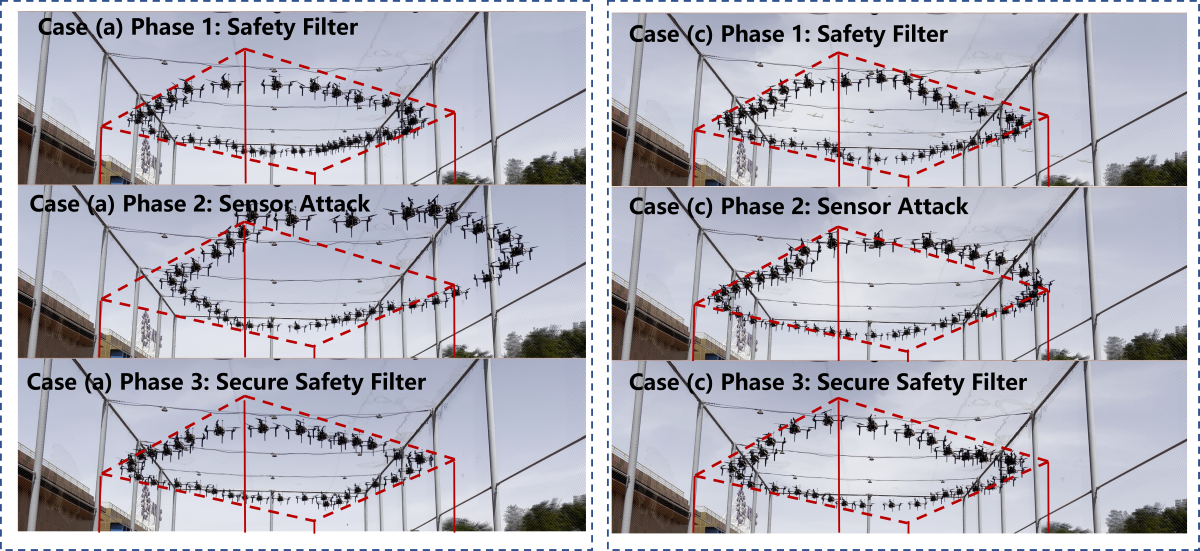} 
         \caption{Snapshots of the drone trajectories with the constant-value attack (Case (a)) and the scale attack (Case (c)).}
    \end{subfigure}
    \begin{subfigure}[t]{\linewidth} 
         \includegraphics[width=1\linewidth]{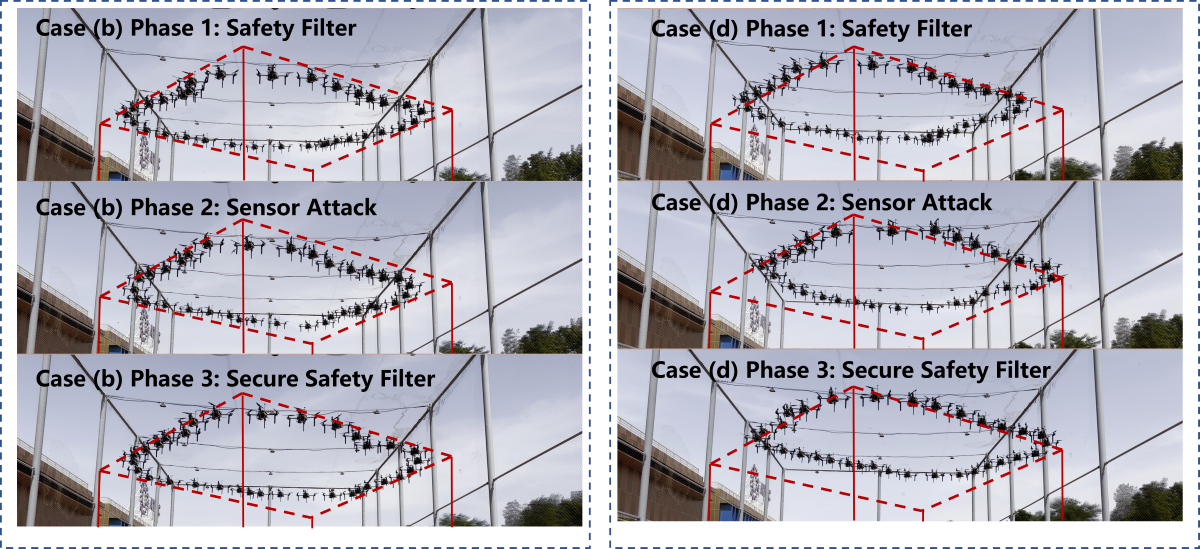} 
         \caption{Snapshots of the drone trajectories with the noise attack (Case (b)) and the shift attack (Case (d)).}
    \end{subfigure}
   
    \caption{Quadrotor trajectories in different phases of the hardware experiments under four different attacks. For easier illustration, all the attacks are performed on $y_1$ or $y_5$ measuring $\bm{x}$-axis position. \textbf{Case (a) Constant value attack} - the measurement  is set to be constant $2.5$; \textbf{Case (b) Noise attack} - A Gaussian noise with zero mean and standard deviation $0.75$ is added to the measurement; \textbf{Case (c) Scale attack} -  the measurement is set to be $0.25x_1$;  \textbf{Case (d) Shift attack} - the measurement is set to be $x_1 + 1.5$.  }
    \label{fig:hw-results}
    \vspace{-20pt}
\end{figure*}

\subsection{Hardware Experiments}

The hardware test utilizes a Holybro X500 quadrotor with minimal modifications on its landing gears to reduce its weight. See Fig.~\ref{fig:drone} for an illustration and its onboard components. Among the two onboard Flight Controllers (FCs), the primary one connects to and provides low-level control signals to the motors. The secondary FC has a separate sensing module providing the second set of position and velocity readings. The firmware flashed into the two FCs is the PX4-Autopilot \hyperlink{https://github.com/PX4/PX4-Autopilot/tree/release/1.14}{release/1.14}.

Figure \ref{fig:hardware-experiment} presents the software architecture of this experiment. As observed, the primal FC (FC1) sends the low-level control commands to the drone's actuators. Each of the two FCs connects to a separate H-RTK GPS receiving GPS corrections for high-accuracy ($\approx\pm2cm$) positioning. The attacker runs in MC and injects the same types of attacks described in Sec.~\ref{h:sitl-experiment} to the second EKF information. As a first step to reproduce our SITL simulation results, for the experimental results reported below, the EKF2 information is an exact duplicate of EKF1 from the primary FC occurring at the ROS2 node interface between the FC and MC. One of these two duplicate signals is then attacked. Future work will test different sensing and attacking configurations, e.g., two independent EKFs, and physical attacks on onboard sensors.

The hardware test follows a similar procedure as in previous subsection. We test four types of attacks, and for each attack we test the performance of the drone in three phases as shown in Fig.~\ref{fig:three-phases}. The drone is commanded to follow a quadrilateral, planar, cyclic path of $(2.8, 2.8)\rightarrow(-2.8, 2.8)\rightarrow(-2.8, -2.8)\rightarrow(2.8, -2.8)$. The safety boundary is:
\begin{equation}
    x_1 \in [-3, 3], \ x_2 \in [-3, 3].
\end{equation}
The results are presented in Fig.~\ref{fig:hw-results}. In the experiments, we see very similar behaviors as in the SITL simulation.  When the constant value attack or the scale attack occur, the safety filter does not prevent the drone from entering the unsafe region. With our proposed secure safety filter, the drone regains safety and resumes its travel as if it were attack-free. We also observe a distinguishable difference in the shift attack case: now the drone flies out of the safe region under the attack. In this case, our secure safety filter still provides safety despite the trajectory becomes conservative on the right side.

\section{Conclusions}
\label{sec:conclusion}
In this work, we proposed a secure safety filter framework for guaranteeing the safety of robotic systems under severe sensor spoofing attacks. In particular, we extended the existing secure state reconstruction results to cope with bounded measurement noise and dealt with the plausible states with a robust safety filter. We tested the effectiveness of our secure safety filter in software-in-the-loop simulations and hardware experiments on a nonlinear drone platform. The tests show that our secure safety filter correctly guards the drone away from unsafe regions even though one of the two units of sensors is under attack. 

\balance

\bibliographystyle{IEEEtran}
\bibliography{IEEEabrv,references}

\end{document}